\documentclass[letterpaper, 10 pt, conference, nofonttune]{ieeeconf}  %
\IEEEoverridecommandlockouts                              
\usepackage[bookmarks=true]{hyperref}
\usepackage{color}
\usepackage{times}
\usepackage{epsfig}
\usepackage{epstopdf}
\usepackage{graphicx}
\usepackage{amsmath}
\usepackage{amssymb}
\usepackage{bbm}
\usepackage[ruled,vlined,linesnumbered]{algorithm2e}
\usepackage{graphicx} 
\usepackage{subfigure} 
\usepackage{subfloat}
\usepackage{verbatim}
\usepackage{booktabs}
\usepackage{multirow}
\usepackage{makecell}
\usepackage{subfigure}
\usepackage{colortbl,booktabs}
\usepackage{lscape}
\usepackage{url}
\usepackage{diagbox}
\usepackage{makecell}
\usepackage{xcolor,colortbl}
\usepackage{float}
\usepackage[english]{babel}

\usepackage{fixltx2e}

\usepackage{microtype}
\usepackage{wrapfig}

\title{\LARGE \bf
Robot Adaptation for Generating Consistent Navigational \\Behaviors over Unstructured Off-Road Terrain}
\author{Sriram Siva$^{1}$, Maggie Wigness$^{2}$, John G. Rogers$^{2}$, and Hao Zhang$^{1}$
\thanks{$^{1}$Sriram Siva and Hao Zhang are with the Human-Centered Robotics Lab in the Department of Computer Science,
        Colorado School of Mines, Golden, CO 80401, USA.
        {Email: \{sivasriram, hzhang\}@mines.edu}.%
        }
 \thanks{$^{2}$Maggie Wigness and John Rogers are with the US Army Research Laboratory (ARL), Adelphi, MD 20783, USA.  {Email: \{maggie.b.wigness.civ, john.g.rogers59.civ\}@mail.mil}.
}
 \thanks{This work was supported by NSF CAREER Award IIS-1942056
and ARL SARA Program W911NF-20-2-0107.
The authors also would like to thank Mr. Eric Spero at ARL for the discussion of this work and for facilitating the experimentation.
}
}

\newcommand{\HZ}[1]{\noindent\textbf{\color{red} *Zhang: #1}}

\begin{document}

\maketitle
\thispagestyle{empty}
\pagestyle{empty}

\begin{abstract}
Terrain adaptation is an essential capability for a ground robot to effectively traverse unstructured off-road terrain in real-world field environments such as forests.
However, the expected robot behaviors generated by terrain adaptation methods cannot always be executed accurately
due to setbacks such as wheel slip and reduced tire pressure.
To address this problem,
we propose a novel approach for consistent behavior generation that enables the ground robot's actual behaviors to more accurately match expected behaviors
while adapting to a variety of unstructured off-road terrain.
Our approach learns offset behaviors that are used to compensate for the inconsistency between the actual and expected behaviors without requiring the explicit modeling of various setbacks.
Our approach is also able to estimate the importance of the multi-modal features to improve terrain representations for better adaptation.
In addition, we develop an algorithmic solver for our formulated regularized optimization problem,
which is guaranteed to converge to the global optimal solution.
To evaluate the method,
we perform extensive experiments using various unstructured off-road terrain in real-world field environments.
Experimental results have validated that our approach enables robots to traverse complex unstructured off-road terrain with more navigational behavior consistency, and it outperforms previous methods, particularly so on challenging terrain.

\end{abstract}

\color{black}
\section{Introduction}

Over the past several years, autonomous ground robots have been increasingly deployed in off-road field environments to address real-world applications, including disaster response, homeland defense, and planetary exploration~\cite{chiang2020safety, bruzzone2012robots}.
Field environments are challenging for ground robots to navigate over
because the terrain is unstructured
and cannot be fully modeled beforehand~\cite{silver2010learning},
and also the
environment typically exhibits a wide variety of characteristics (e.g, changing terrain types, and varying slope and roughness), as demonstrated in Fig. \ref{Motivation}.
While operating in a field environment,
robots need to generate effective navigational behaviors for traversing over the off-road terrain in order to successfully complete navigation tasks.
Robot terrain adaptation is then an essential capability for robots to adapt their navigational behaviors according to unstructured off-road terrains~\cite{siva2019robot}.

Given its importance, the research problem of robot terrain adaptation has been widely investigated over the past several years.
Previous learning-based methods can be divided into two broad categories: terrain classification and terrain adaptation.
The first category of approaches apply
a \color{black}
robot's exteroceptive and proprioceptive sensory data to classify  terrain types and estimate traversability for a robot to navigate over the terrain~\cite{bermudez2012performance, brooks2005vibration, devjanin1983six, dupont2008terrain, hudjakov2009aerial}.
This category also includes methods that model terrain complexity for planning robot navigation tasks~\cite{silver2010learning,peynot2014learned}.
The second category of methods focus on directly generating adaptive navigational behaviors according to terrain in order to successfully complete navigation tasks~\cite{silver2010learning,han2017sequence,pastor2009learning,wigness2018robot}.
Specifically, learning from demonstration (LfD) is widely used to transfer human expertise to robots in order to achieve human-level robot navigational control~\cite{siva2019robot,wigness2018robot,ge2015learning}.

\begin{figure}[t]
\centering
\vspace{6pt}
\includegraphics[width=0.48\textwidth]{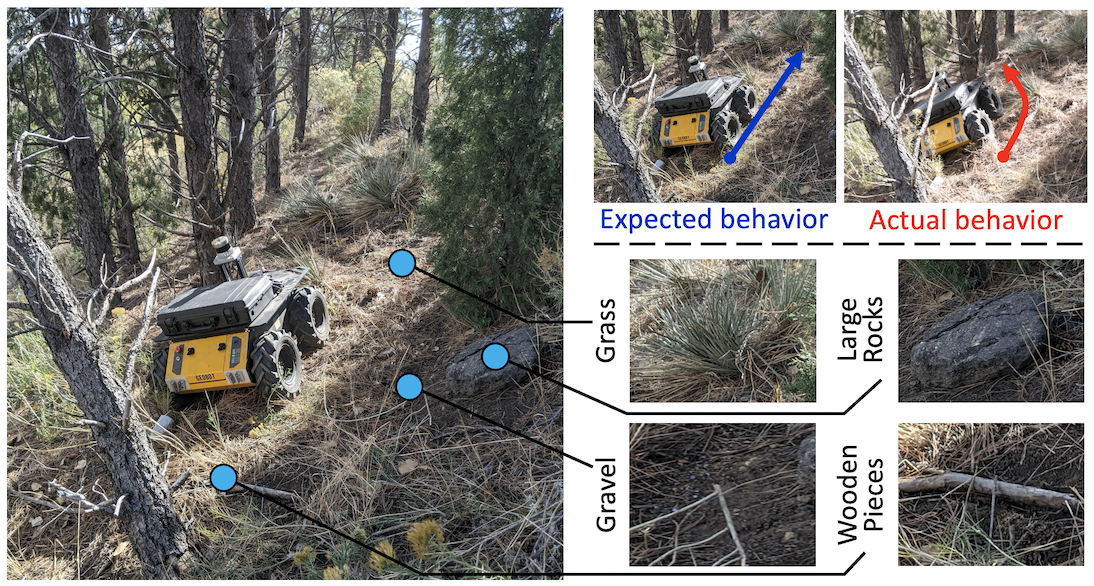}
\caption{This motivating scenario illustrates the necessity for ground robots to generate consistent navigational behaviors 
 while operating in real-world field environments. 
The off-road environment is unstructured and exhibits a variety of characteristics, including changing terrain types and slopes. 
When ground robots are deployed in this environment, 
their actual behaviors often do not match the expected behaviors, e.g., due to wheel slip.
Thus, the capability of consistent behavior generation is essential for ground robots while they navigate over unstructured off-road terrain.
}\label{Motivation}
\end{figure}

Despite their promising performance in complex off-road environments,
the expected navigational behaviors generated by these approaches
cannot always be executed accurately,
i.e., the actual robot navigational behaviors are not consistent with the expected behaviors.
This inconsistency is mainly caused by setbacks~\cite{knight2001balancing, borges2019strategy} that are defined as factors that increase the difficulty for a robot to achieve its expected navigational behaviors. Example setbacks include wheel slip, heavy robot payload,
 and reduced tire pressure.
\color{black}
Existing learning-based approaches (e.g., LfD) for robot navigation generally ignore
these
setbacks,
which often leads to
\color{black}
not being able to consistently complete the learned navigational behaviors
when deployed
\color{black}
over unstructured off-road terrain.
The challenge of how to generate consistent navigational behaviors for robot terrain adaptation has not been well addressed. 



In this paper, we develop a novel approach for consistent behavior generation that enables a robot's actual navigational behaviors to match the expected behaviors
while adapting to a variety of unstructured off-road terrain during navigation.
Our approach learns offset behaviors for the robot to compensate for the inconsistency between the actual and expected navigational behaviors without explicitly modeling the setbacks, while also adaptively navigating over various terrain.
In addition, our approach is able to integrate multi-modal features to characterize terrain
and automatically estimate the importance of these features.
This is all implemented in a unified mathematical regularized optimization framework with a theoretical convergence guarantee.



The novelty of this paper is twofold:
\begin{itemize}
\item We propose a novel formulation to generate consistent navigational behaviors via generating offset behaviors. We also introduce two regularization terms to learn important terrain features and historic behavior differences to enable robot adaptation to unstructured terrain.
\item We propose a new optimization algorithm to address the formulated convex regularized optimization problem with dependent variables,
which holds a theoretical guarantee to effectively converge to the global optimal solution.
\end{itemize}
As an experimental contribution, we provide a comprehensive performance evaluation of learning-based terrain adaptation methods by designing a set of live robot navigation scenarios over a wide variety of individual and complex unstructured off-road terrain.

The remainder of the paper is organized as follows. A
review related work is provided in Section \ref{sec:related_work}. The proposed approach is discussed
in Section \ref{sec:RTA}.
After describing the experimental results in Section \ref{sec:EXPT}, 
we conclude the paper in Section \ref{sec:CONC}.

\section{Related Work}\label{sec:related_work}

In this section, we provide a review of existing methods of terrain classification and robot adaptation that have been used to enhance autonomous navigation.

\subsection{Terrain Classification}
Terrain classification methods use sensory data from a robot to classify the terrain.  
Many earlier methods were developed to address the specific needs of larger vehicles \cite{jumikis1970introduction}. 
For example, in the DARPA Grand Challenge, Urmson et al. \cite{urmson2004high} designed a pre-selected set of optimal speeds for each terrain,
which uses sensory data cues collected over many months to categorize individual types of terrain. 
Some methods used a pre-existing terrain map to achieve high-speed terrain navigation \cite{shimoda2005potential,spenko2006hazard}. Techniques were also designed to use terrain ruggedness data to evaluate navigational behaviors in an online fashion for ground vehicles \cite{brooks2005vibration,sadhukhan2004terrain}. 

Recently, learning-based methods have been used to classify terrains. For example, an SVM classifier was used on terrain features learned from Hidden Markov Models to identify the terrain \cite{trautmann2011mobility}. Color-based terrain classification was performed to generate robot navigational behaviors by labeling obstacles \cite{manduchi2005obstacle}. However, these methods rely on a discrete categorization of terrain types to perform robot navigation.
In real-world field environments, unstructured off-road terrains have a wide variety of characteristics and cannot be easily categorized into distinctive types. 
Thus, these methods typically do not work well to characterize complex unstructured off-road terrain and cannot directly enable robot adaptation to unstructured terrain in real-world field environments.

\subsection{Robot Adaptation}

The second category of methods focus on enabling robots to intelligently adapt to unstructured terrain. The general robot adaptation problem is commonly investigated in robotics \cite{parker1996alliance, parker2000lifelong, he2019underactuated, nikolaidis2017human, papadakis2013terrain} using high-level behavior models. Case-based reasoning was used to adapt a robot in rapidly changing environments \cite{watson1994case}. Earlier works considered ground speed as the optimization variable and formulated a method for trading {progress} and {velocity} with changing terrain characteristics \cite{fox1997dynamic}.
Recently, an array of inertial and ultrasonic sensors were used to recognize soil properties that allow robots to perform terrain adaptation effectively \cite{nabulsi2006multiple}.

Learning based methods for terrain adaptation have gained significant attention because of their effectiveness and flexibility \cite{thrun1998lifelong}. Early work addressed terrain adaptation from the perspective of online learning that updates model parameters in the execution phase \cite{kleiner2002towards}. Although successful, this approach lacks the ability to fast adapt on the fly to deal with sudden terrain changes. To overcome this challenge, methods were developed that generate navigational behaviors according to the predicted terrain characteristics \cite{plagemann2008learning}. Similarly, reinforcement learning based navigation was developed to generate stable locomotion patterns for terrain adaptation \cite{erden2008free}. Methods based on inverse reinforcement learning were used to mimic expert navigation controls to achieve human-level maneuverability \cite{wigness2018robot}. More recently, a unified method for simultaneous terrain classification and apprenticeship learning showed promising results on robot adaptation to unstructured terrain \cite{siva2019robot}.

Given the promise of the previously mentioned learning-based terrain adaptation techniques, almost all methods focus solely on learning the expected behaviors, without addressing the problem of generating consistent navigation behaviors. 
Our approach is introduced to address this challenge that has not been well addressed by the existing learning-based robot terrain adaptation methods.



\color{black}

\section{Approach}\label{sec:RTA}


In this section, we discuss the proposed method for consistent navigational
behavior generation and our algorithm to solve the formulated regularized optimization problem.
\color{black}
\subsection{Problem Formulation}\label{sec:RAL-A}

As a robot traverses over terrain,
at each time step, we extract multi-modal features from observations acquired from multiple sensors installed on the robot (e.g., visual camera, LiDAR, and IMU).
We concatenate all features extracted at time point $t$ into a vector and denote it as $\mathbf{x}^{(t)} \in \mathbb{R}^q$, where $q=\sum_{j=1}^{m} q_j $, $q_j$ is the dimensionality of the $j$-th feature modality, and $m$ is the number of  modalities. 
We represent features extracted from a sequence of consecutive $c$ time points as a data instance
$\mathbf{x} = [\mathbf{x}^{(t)};\dots;\mathbf{x}^{(t-c)}] \in \mathbb{R}^{d}$, where $d=c \times q$.
We further denote the set of $n$ data instances for training our approach
as $\mathbf{X} = [\mathbf{x}_{1}, \dots , \mathbf{x}_{n}] \in \mathbb{R}^{d \times n}$.

We use $\mathbf{Y} =[\mathbf{y}_{1}, \dots , \mathbf{y}_{n}]  \in \mathbb{R}^{r \times n}$ to denote the expected navigational behaviors of the robot associated with $\mathbf{X}$,
where $\mathbf{y}_{i} \in \mathbb{R}^{r}$ is a vector of $r$ behavior control variables corresponding to  $\mathbf{x}_{i}$. These variables represent the behaviors that the robot is expected to execute (e.g., velocity, motor torque, and steering angle) when traversing over the terrain.
Due to momentum, we know the robot's behaviors executed at the present time are dependent on previous time steps, so
we estimate the robot's behaviors $\mathbf{y}_i$
using  $\mathbf{x}_{i} = [\mathbf{x}_{i}^{(t)};\dots;\mathbf{x}_{i}^{(t-c)}]$,
 taking into account the history of $c$ observations.
Then, the problem of
navigational behavior estimation can be formulated as:
\begin{equation}\label{eq0}
\min_{\mathbf{W}} \Vert \mathbf{Y} - \mathbf{W}^\top \mathbf{X}  \Vert_{F}^{2} + \lambda_{1}\Vert \mathbf{W} \Vert_{M}
\end{equation}
where $\mathbf{W} \in \mathbb{R}^{d \times r}$  is a weight matrix that is illustrated in Fig. \ref{weight_martix_w}.
Each element $\mathbf{w}^{j(k)}_{i} \in \mathbf{W}$ indicates the importance of the $i$-th feature modality in $\mathbf{x}^{(k)}$ for generating the $j$-th navigational behavior.

\begin{figure}[tb]
\centering
\includegraphics[width=0.48\textwidth]{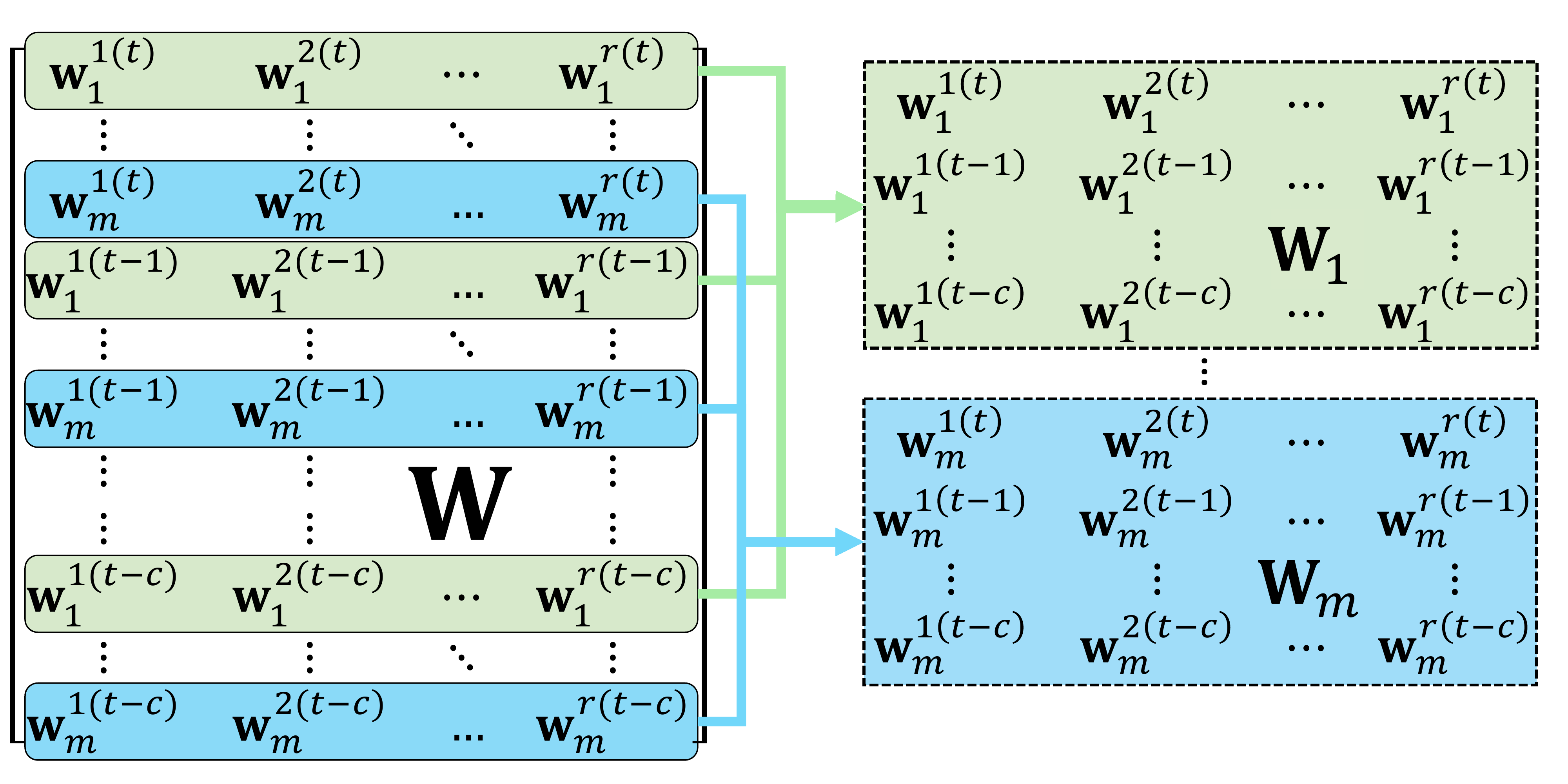}
\caption{Illustration of the weight matrix $\mathbf{W}$, which can be decomposed into $[\mathbf{W}_1,\dots,\mathbf{W}_m]$.
Each $\mathbf{W}_i$ (e.g., the green and blue matrices on the right) represents the importance of the $i$-th feature modality.}
\label{weight_martix_w}
\end{figure}

The first term in Eq. (\ref{eq0}) is a loss function that encodes the error of using a linear model parameterized by $\mathbf{W}$ to project $\mathbf{X}$ to $\mathbf{Y}$.
The second term is a regularization term named
 the
feature modality norm and is mathematically defined as:
\begin{equation}
\Vert \mathbf{W} \Vert_{M} = \sum_{i=1}^{m}\Vert \mathbf{W}_{i} \Vert_{F} =   \sum_{i=1}^{m}\sqrt{\sum_{j=1}^{r}\sum_{k=t}^{t-c}(\mathbf{w}^{j(k)}_{i})^\top( \mathbf{w}^{j(k)}_{i})}
\end{equation}
where $\Vert . \Vert_{F}$ is the Frobenius norm.
The feature modality norm groups together weights within a feature modality and enforces sparsity among different modalities, thus,
identifying the most descriptive features for behavior generation.
This is a critical capability for
 ground  robot
navigation
\color{black}
since different features typically capture different characteristics of the unstructured terrains
(e.g., color, slope, and roughness),
and have different effects toward generating navigational behaviors.
The trade-off hyperparameter $\lambda_{1}$ in Eq. (\ref{eq0}) is used to balance the loss and the regularization term.

The problem formulation in Eq. (\ref{eq0}) allows ground robots to adapt their navigational behaviors according to different terrain features.
However, due to setbacks that reduce the effectiveness of robot navigation,
such as wheel slip, heavy payload, and reduced tire pressure~\cite{sidek2013exploiting},
the robot's actual behaviors do not match the expected behaviors.

\color{black}

\subsection{Consistent Navigational Behavior Learning}

The key novelty of this paper is that we propose a principled method for
ground robots to generate consistent navigational behaviors that adapt to unstructured terrain.
Our approach monitors the difference between the actual and the expected navigational behaviors caused by the setbacks,
and computes an offset to reduce
the difference. This allows our approach to achieve consistent robot behaviors
without the requirement of
explicitly
modeling all the setbacks.
\color{black}



Formally, we denote the actual behaviors executed by the robot as $\hat{\mathbf{Y}} =[\hat{\mathbf{y}_{1}}, \dots , \hat{\mathbf{y}_{n}}]  \in \mathbb{R}^{r\times n}$, where $\hat{\mathbf{y}}_{i} \in \mathbb{R}^{r}$ denotes the actual behaviors executed by the robot when observing $\mathbf{x}_{i}$.
We define that the actual behaviors $\hat{\mathbf{y}_{i}}$ is composed by
the expected behaviors $\mathbf{y}_{i}$ and the offset behaviors $\mathbf{v}_{i} \in \mathbb{R}^{r}$,
i.e., $\hat{\mathbf{y}_{i}} = \mathbf{y}_{i} + \mathbf{v}_{i}$.
The offset behaviors $\mathbf{v}_{i}$ is computed as $\mathbf{v}_{i} = \mathbf{U}^{\top}\mathbf{e}_{i}$,
where $\mathbf{e}_{i} = [(\hat{\mathbf{y}}_{i}^{(t)}-\mathbf{y}^{(t)}_{i});\dots;(\hat{\mathbf{y}}^{(t-c)}_{i}-\mathbf{y}^{(t-c)}_{i})] \in \mathbb{R}^{rc}$
denotes a vector of differences between actual and expected behaviors
in the
previous
$c$ time steps.
$\mathbf{U} = [\mathbf{u}^1,\dots,\mathbf{u}^r] \in \mathbb{R}^{rc \times r}$ is the weight matrix,
and $\mathbf{u}^{j} \in \mathbb{R}^{rc}$ indicates the importance of $\mathbf{e}_{i}$ towards generating the $j$-th element in $\mathbf{v}_{i}$.
Using data from a history of $c$ time steps allows our method to consider inertia (i.e., resistance to change in behaviors) during navigation.

Then, generating consistent navigational behaviors can be formulated as:
\begin{equation}\label{eq1}
\min_{\mathbf{U},\mathbf{W}} \Vert \hat{\mathbf{Y}} - \mathbf{W}^\top \mathbf{X} - \mathbf{U}^\top \mathbf{E}  \Vert_{F}^{2} + \lambda_{1}\Vert \mathbf{W} \Vert_{M}
\end{equation}
where $\mathbf{E} = [\mathbf{e_{1}},\dots,\mathbf{e}_{n}] \in \mathbb{R}^{rc \times n}$.
The loss function models the actual behavior by considering both $\mathbf{X}$ and $\mathbf{E}$ to achieve consistent navigational behaviors.
Because of inertia, historical data from different past time steps may contribute differently towards generating offset behaviors
(e.g., a heavier robot with bigger inertia often needs to consider a longer history).
Thus,
we propose a new regularization term to explore which time steps in the historical data
are
more important
for generating the offset behaviors.
We name this regularization term
the
temporal norm, which is expressed by:
\begin{equation}
\Vert \mathbf{U} \Vert_{T} =  \sum_{k=t}^{t-c}\Vert \mathbf{U}^{(k)} \Vert_{F}=\sum_{k=t}^{t-c}\sqrt{\sum_{j=1}^{r}\|{\mathbf{u}}^{j(k)}\|_2^2}
\end{equation}
where $\mathbf{U}^{(k)} = [\mathbf{u}^{1(k)},\dots,\mathbf{u}^{r(k)}] \in \mathbb{R}^{r \times r}$ is the weight matrix,
and $\mathbf{u}^{j(k)}$ indicates the importance of the behavior difference $\mathbf{e}^{(k)}$
for generating the offset $\mathbf{v}^{(k)}$ at the $k$-th time step.
This norm groups together weights for the vector of behavior differences at each time step,
and enforces sparsity between weights at different time steps to identify the most important time steps.


\color{black}

Using both norms
to generate consistent robot navigational behaviors while identifying important feature modalities and historical time steps,
the final objective function becomes:
\begin{equation}\label{final}
\min_{\mathbf{U},\mathbf{W}} \Vert \hat{\mathbf{Y}} - \mathbf{W}^\top \mathbf{X}  - \mathbf{U}^\top {\mathbf{E}} \Vert_{F}^{2} + \lambda_{1}\Vert \mathbf{W} \Vert_{M} + \lambda_{2}\Vert \mathbf{U} \Vert_{T}
\end{equation}
where $\lambda_1 \geq 0$ and $\lambda_2 \geq 0$ are the trade-off hyper-parameters to balance the loss function and regularization terms.


\subsection{Generating Consistent Behaviors during Execution}

After computing the optimal values of the weight matrices $\mathbf{W}$ and $\mathbf{U}$ according to Algorithm \ref{alg1},
at each time step during execution,
our approach extracts a feature vector  $\mathbf{x} \in \mathbb{R}^{d}$ from
the
robot's new observation.
Then, it computes its corresponding offset behavior by $\mathbf{v} = \mathbf{U}^{\top}\mathbf{e}$,
where $\mathbf{e}$ denotes the difference between the expected behavior $\mathbf{y}$ and the actual behavior $\hat{\mathbf{y}}$ (e.g., measured using a pose estimation technique based upon SLAM or visual odometry).


In addition, we predict the offset needed for the next time step so our approach can proactively generate behaviors that take
into account future behavior differences.
The predicted offset behavior  $\tilde{\mathbf{v}}$ is estimated by\footnote{\label{supplementary}Derivation is presented in the supplementary material at:
\url{http://hcr.mines.edu/publication/TerrainAdapt_Supp.pdf}}:
\begin{equation}
\tilde{\mathbf{v}} = \sum_{k=t}^{t-c}\Big(\big(\mathbf{U}^{(k)\top}\big)^{-1}\big(\mathbf{y}^{(k)} - \mathbf{W}^{(k)\top}\mathbf{x}^{(k)}\big) \Big)
\end{equation}

Including both offset behaviors for the current and future time steps,
our approach allows the robot to generate
consistent actual navigational behaviors by:
\begin{equation}
\mathbf{y} = \mathbf{W}^{\top}\mathbf{x} +
\begin{bmatrix}
\mathbf{I}_{r}\\
\mathbf{U}
\end{bmatrix}
^{\top}
\begin{bmatrix}
\tilde{\mathbf{v}} \\
\mathbf{e}
\end{bmatrix}
\end{equation}
where $\mathbf{I}_{r} \in \mathbb{R}^{r \times r}$ is an identity matrix.



\begin{algorithm}[t]
    \SetKwInOut{Input}{Input}
    \SetKwInOut{Output}{Output}
    \SetKwInOut{return}{return}

    \Input{$\mathbf{X}\in\mathbb{R}^{d\times n}$, $\mathbf{Y}\in\mathbb{R}^{r\times n}$, and
    $\mathbf{E}\in\mathbb{R}^{rc\times n}$}

    \Output{The weight matrices $\mathbf{W}\in \mathbb{R}^{d \times r}$ and $\mathbf{U}\in \mathbb{R}^{rc \times r}$}

    Initialize $\mathbf{W}\in \mathbb{R}^{d \times r}$ and $\mathbf{U}\in \mathbb{R}^{rc \times r}$;

    \While{not converge}
    {
    Calculate the block diagonal matrix $\mathbf{Q}$ with $i$-th diagonal block given as $\frac{1}{2\Vert \mathbf{W}_{i}\Vert_{F}}\mathbf{I}_{q_i}$;

    Compute the matrix $\mathbf{W}$ according to Eq. (\ref{optW2});

    Calculate the block diagonal matrix $\mathbf{P}$ with the $k$-th diagonal block as $\frac{1}{2\Vert \mathbf{U}^{(k)}\Vert_{F}}\mathbf{I}_{r}$;

    Compute the matrix $\mathbf{U}$ according to Eq. (\ref{optZ2});
    }
    \textbf{return: } $\mathbf{W}\in \mathbb{R}^{d \times r}$ and $\mathbf{U}\in \mathbb{R}^{rc \times r}$

 \caption{The proposed algorithm to solve the formulated regularized optimization problem in Eq. (\ref{final}).} 
 \label{alg1}
\end{algorithm}

\subsection{Optimization Algorithm}

As the second contribution,
we design an iterative optimization algorithm to obtain the optimal solution to the proposed regularized optimization problem in Eq. (\ref{final}).
The optimization problem is challenging to solve in general
because of the two non-smooth structured regularization terms.
The algorithm is presented in Algorithm \ref{alg1}.

To solve for the optimal weight matrix $\mathbf{W}$, we minimize Eq. (\ref{final}) with respect to $\mathbf{W}$, resulting in:
 \begin{equation}\label{optW1}
  2\mathbf{X}^\top \mathbf{X} \mathbf{W} - 2\mathbf{X}\hat{\mathbf{Y}} + 2\mathbf{XU}^\top\mathbf{E}  + \lambda_{1}\mathbf{Q}\mathbf{W} = 0
 \end{equation}
where $\mathbf{Q}$ is a diagonal matrix with the $i$-th diagonal block given as $\frac{1}{2\Vert \mathbf{W}_{i}\Vert_{F}}\mathbf{I}_{q_i}$ and $\mathbf{I}_{q_i}$ is an identity matrix.
Then, we compute $\mathbf{W}$ in a closed-form solution by:
 \begin{equation}\label{optW2}
\mathbf{W} =(  2\mathbf{X}^\top \mathbf{X}  + \lambda_{1}\mathbf{Q})^{-1} (2\mathbf{X}\hat{\mathbf{Y}} - 2\mathbf{XU}^\top\mathbf{E})
 \end{equation}
Because $\mathbf{Q}$ and $\mathbf{W}$ are
interdependent, an iterative algorithm is required to compute them.  \color{black}

To compute $\mathbf{U}$, we calculate the derivative of the objective function in Eq. (\ref{final}) and set the equation to zero as:
 \begin{equation}\label{optZ1}
  2\mathbf{E}^\top \mathbf{E} \mathbf{U} - 2\mathbf{E}\hat{\mathbf{Y}} + 2\mathbf{EW}^\top\mathbf{X}  + \lambda_{2}\mathbf{P}\mathbf{U} = 0
 \end{equation}
where $\mathbf{P}$ is a diagonal matrix with the $k$-th diagonal block as $\frac{1}{2\Vert \mathbf{U}^{(k)}\Vert_{F}}\mathbf{I}_{r}$, with $\mathbf{I}_{r}$ being an identity matrix.
Then, we compute $\mathbf{U}$ in a closed-form solution as:
 \begin{equation}\label{optZ2}
  \mathbf{U} = (2\mathbf{E}^\top \mathbf{E}    + \lambda_{2}\mathbf{P})^{-1} (2\mathbf{E}\hat{\mathbf{Y}} - 2\mathbf{EW}^\top\mathbf{X})
 \end{equation}
This $\mathbf{U}$ is then used to calculate $\mathbf{W}$ in the next iteration.


\noindent \textbf{Convergence.} 
Algorithm \ref{alg1} is guaranteed to converge to the global optimal solution in the formulated regularized optimization problem in Eq. (\ref{final}).
The proof is provided in the supplementary material$^{\ref{supplementary}}$.

\noindent \textbf{Complexity.} As the formulated optimization problem in Eq. (\ref{final}) is convex, Algorithm \ref{alg1} converges fast (e.g., within tens of iterations only).
In each iteration of our algorithm, computing Steps 3 and 5 is trivial. Steps 4 and 6 can be computed by solving a system of linear equations with quadratic complexity. 

\section{Experiments}\label{sec:EXPT}
This section presents the experiment and implementation setups, and presents an analysis of experimental results obtained by our approach and previous terrain adaptation methods.

\subsection{Experimental Setups}

\begin{table*}[htb]
\centering
\caption{Quantitative results based on ten runs for scenarios when the robot traverses over \textbf{individual types of unstructured terrain} shown in Fig. \ref{Scenario_1}.
Successful runs (with no failures) are used to calculate the metrics of traversal time, inconsistency and jerkiness.
Our approach is compared with LfD~\cite{ge2015learning}, MM-LfD~\cite{wu2018multi} and TRAL~\cite{siva2019robot}.}
\label{tab:RTT1}
\tabcolsep=0.1cm
\begin{tabular}{ | c| c| c| c| c|| c| c| c| c|| c| c| c| c|| c| c| c| c|}
\hline
\multicolumn{1}{|c|}{} & \multicolumn{4}{c||}{Failure Rate (/10) } & \multicolumn{4}{c||}{Traversal Time (s)} & \multicolumn{4}{c||}{Inconsistency}& \multicolumn{4}{c|}{Jerkiness (m/s$^3$)} \\
\hline Terrain & LfD & MM-LfD &  TRAL & \textbf{Ours} & LfD  & MM-LfD & TRAL  & \textbf{Ours}& LfD  & MM-LfD & TRAL & \textbf{Ours}& LfD  & MM-LfD & TRAL  & \textbf{Ours}\\ \hline
Grass  & 0 & 0 & 0 &  $\mathbf{0}$ &
17.9 & 18.2 & $\mathbf{17.4}$ &  17.5 &
2.82 & 3.06 & $\mathbf{1.84}$ &  2.11 &
79.59 & 81.42 & $\mathbf{75.18}$ & 76.50 \\

Sand & 0 & 1 & 0 &  $\mathbf{0}$ &
15.3 & $\mathbf{12.7}$ & 15.9 &  13.1 &
4.72 & 4.62 & 4.67 &  $\mathbf{4.55}$  &
71.14 & 73.74 & 65.32 &  $\mathbf{64.22}$\\

Gravels & 0 & 0 & 0 & $\mathbf{0}$ &
 22.9 & 24.1 & 20.4 & $\mathbf{20.2}$ &
4.12 & 4.63 & 3.81 &  $\mathbf{3.04}$ &
44.95 & 48.27 & 40.26 & $\mathbf{39.48}$\\

M.Rock & 1 & 3 & 0 &  $\mathbf{0}$ &
33.2 & 36.9 & 29.4 &  $\mathbf{28.4}$ &
7.92 & 9.59 & 4.21 &  $\mathbf{2.41}$ &
141.48 & 144.22 & 113.34 & $\mathbf{111.18}$ \\

L.Rock & 6 & 6 & 2 & $\mathbf{1}$ &
57.8 & $\mathbf{55.3}$ & 63.4 &  60.9 &
24.79 & 28.50 & 9.51 & $\mathbf{7.84}$ &
52.55 & 54.30 & 49.50 & $\mathbf{48.36}$ \\
\hline
\end{tabular}
\end{table*}

\begin{table*}[htb]
\centering
\caption{Quantitative results for scenarios when the robot traverses over
\textbf{complex unstructured off-road terrain} shown in Fig. \ref{Scenario_2}. 
Successful runs (with no failures) are used to calculate the metrics of traversal time, inconsistency and jerkiness.}
\label{tab:RTT2}
\tabcolsep=0.1cm
\centering
\begin{tabular}{ | c| c| c| c| c|| c| c| c| c|| c| c| c| c|| c| c| c| c|}
\hline
\multicolumn{1}{|c|}{} & \multicolumn{4}{c||}{Failure Rate (/10)}& \multicolumn{4}{c||}{Traversal Time (s)}  & \multicolumn{4}{c||}{Inconsistency}& \multicolumn{4}{c|}{Jerkiness (m/s$^3$)}\\
\hline
Terrain & LfD  & MM-LfD  &  TRAL & \textbf{Ours} & LfD   & MM-LfD & TRAL & \textbf{Ours}& LfD   & MM-LfD & TRAL   & \textbf{Ours}& LfD  & MM-LfD  & TRAL  & \textbf{Ours}\\
\hline

Gr-M.Rock & 5 & 7 & 2 & $\mathbf{1}$ &
22.0 & $\mathbf{19.7}$ & 27.5 & 23.1 &
15.62 & 17.28 & 14.54 & $\mathbf{12.31}$ &
65.01 & 80.56 & 58.36 & $\mathbf{51.93}$ \\

Gr-L.Rock & 8 & 9 & 3 &  $\mathbf{3}$ &
$\mathbf{27.2}$ & 27.4 & 29.4 & 28.8 &
93.53 & 101.26 & 68.87 &  $\mathbf{51.16}$&
34.96 & 40.51 & 28.22 & $\mathbf{24.55}$ \\

M.Terrain I  & 0 & 1 & 0 &  $\mathbf{0}$ & 
$\mathbf{17.9}$ & 18.2 & 19.4 &  18.9 &
3.97 & 5.38 & 4.91 & $\mathbf{3.39}$ &
72.37 & 83.17 & 70.36 &  $\mathbf{68.55}$ \\ 

M.Terrain II & 5 & 7 & $\mathbf{4}$ &  5 &
23.1 & $\mathbf{18.1}$ & 30.2 &  28.5 &
93.37 & 95.47 & 80.43 &  $\mathbf{78.82}$ &
54.13 & 77.49 & 52.51 &  $\mathbf{47.93}$ \\

\hline
\end{tabular}
\end{table*}

We use the Clearpath Husky robot in our field experiments.
The robot is equipped with an Intel Reasense D435 color-depth camera and an Ouster OS1-64 LiDAR. The robot also has a variety of sensors to measure its internal states,
including IMU readings, wheel odometry, power consumption, motor speed and battery status.
If sensor readings have a lower frame rate,
linear interpolation is used to get a steady 30 Hz frame rate from all sensors.

\begin{figure}[h]
\centering
\includegraphics[width=0.48\textwidth]{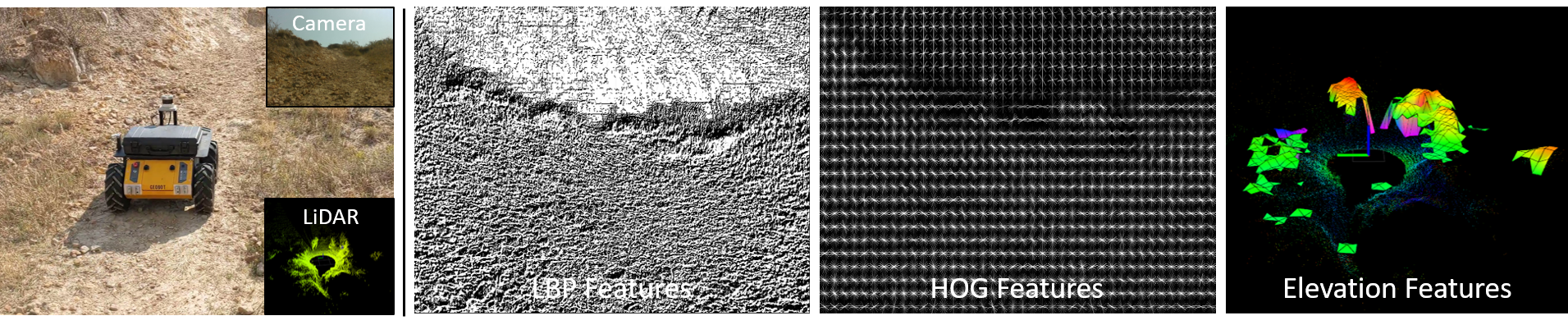}
\caption{Multimodal features used to characterized surrounding terrain.}\label{features}
\end{figure}

To represent unstructured terrain,
we implement multiple visual features extracted from color images to describe different terrain characteristics,
including Histogram of Oriented Gradients (HOG)~\cite{dalal2005histograms} to describe shape and Local Binary Patterns (LBP)~\cite{ahonen2006face} to describe texture.
\color{black}
We also compute an
 elevation map
\color{black} from LiDAR data to represent the grid-wise elevation of the terrain around the robot.
These features are
visualized \color{black}
in Fig. \ref{features}.
During training, expected navigational behaviors are provided by human operators
who control the robot to traverse unstructured terrains as fast as possible
while maintaining safety (e.g., no flipping or crashing).
During training and execution phases,
actual robot navigational behaviors are estimated from LiDAR based SLAM~\cite{legoloam2018}.
We use a sequence of fifteen frames (i.e., $c=15$),
and hyperparameters $\lambda_{1}=0.1$ and $\lambda_{2}=10$ for all experiments.

 \begin{figure}
 \centering
 \includegraphics[width=0.48\textwidth]{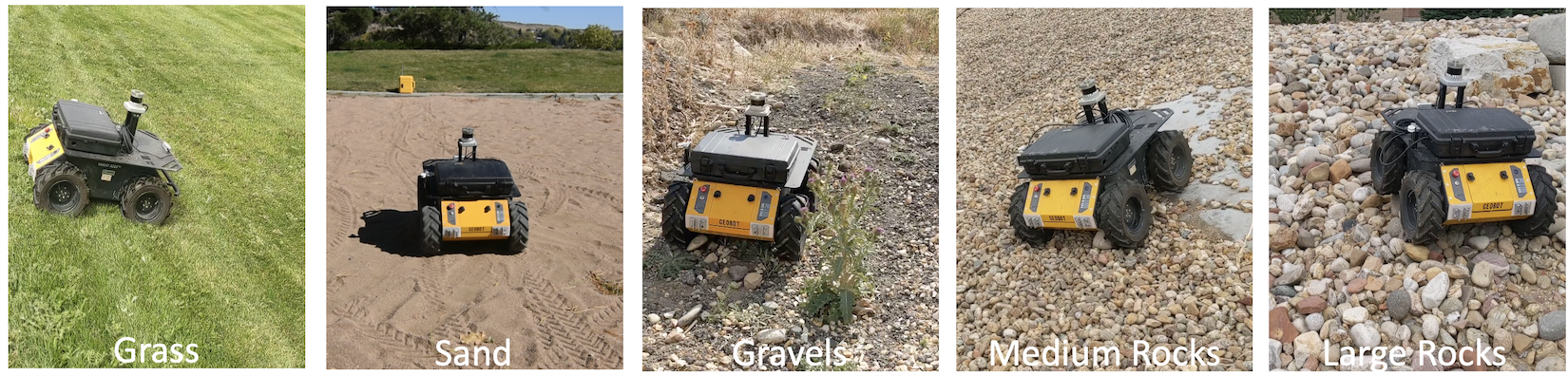}
 \caption{Individual types of unstructured terrain.
 }\label{Scenario_1}
 \end{figure}

\begin{figure}
\centering
\includegraphics[width=0.48\textwidth]{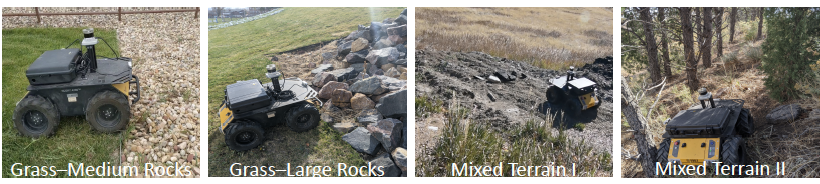}
\caption{Complex unstructured off-road terrain.
}\label{Scenario_2}
\end{figure}

We compare our approach with several previous state of the art learning-based robot navigation techniques,
including Learning from Demonstration (LfD) for robot navigation~\cite{ge2015learning},
multi-modal LfD (MM-LfD)~\cite{wu2018multi},
and Terrain Representation and Apprenticeship Learning (TRAL)~\cite{siva2019robot}.
To comprehensively evaluate the performance of robot navigation in a quantitative fashion,
we use four evaluation metrics:
\begin{itemize}
    \item \emph{Failure Rate (FR)}: This metric is defined as the number of times the robot fails to complete the navigation task across a set of experimental trials.
    If a robot flips or is stopped by a terrain obstacle,
  it is considered a failure. \color{black}
    Lower values of FR indicate better performance. \color{black}
    \item \emph{Traversal Time (TT)}: This metric is defined as the time taken to complete the navigation task over given terrain.
   Smaller values of TT indicate better performance.
    \color{black}
    \item \emph{Inconsistency}: This metric is defined as the error between the expected behavior and the actual behavior in terms of robot poses (linear position and angular). \color{black} 
        Lower values of inconsistency indicate better performance. \color{black}
    \color{black}
    \item \emph{Jerkiness}: This metric is defined as the average sum of the acceleration derivatives along all axes, with lower values indicating better performance. Jerkiness indicates how 
    smooth a robot can traverse over a terrain. 
        Because state estimation and SLAM 
        methods (e.g., based on Kalman filters) may assume smooth robot motions, jerkiness is a useful metric.
\end{itemize}

\subsection{Navigating over Individual Types of Unstructured Terrain}

In this set of experiments,
the robot navigates over individual terrain using off-road tracks.
Each track is made up of one type of terrain and is 
approximately ten meters long.
Five types of terrain are used in our experiments, which are illustrated in Fig. \ref{Scenario_1}.
Our approach is trained on data collected while the robot is manually controlled by an expert to traverse the terrain.
Then, the learned model is deployed on the robot to autonomously navigate over the terrain.
Evaluation metrics for each method are computed across ten trials on each type of terrain track.
\color{black}



The quantitative results achieved by our approach and the comparison to other methods on the metric of failure rate and traversal time are presented in Table \ref{tab:RTT1}.
For simple individual terrains, such as grass and gravel,
all methods allow the robot to successfully traverse over the terrain. However, for more challenging terrain, especially large rocks,
both LfD and MM-LfD have a high failure rate, whereas, TRAL and our approach have a low failure rate.
Our approach only has one failure over the difficult large-rock terrain and outperforms other tested methods.
The presented traversal time is computed by averaging the traversal time across all successful runs, i.e., it excludes the failed trials captured by the FR metric.
It is observed that all methods have a similar traversal time.
In successful runs, both LfD methods show less traversal time compared to other methods over rocky terrain,
although they also have a much higher failure rate.
Thus, the emphasis of high speed traversal used by these methods produces an unreliable system when the robot traverses unstructured off-road terrain in real-world field environments.

Table \ref{tab:RTT1} also presents the quantitative results for the inconsistency and jerkiness metrics.
We observe that both LfD and MM-LfD methods do not perform well and have higher values of inconsistency over individual types of terrain, especially on the large-rock terrain.
TRAL has lower inconsistency and performs the best over the grass terrain.
Our proposed method outperforms the previous approaches and obtains the lowest averaged inconsistency value.
Finally, we also evaluate the tested methods using the jerkiness metric.
An interesting observation is that the medium-rock terrain causes the
largest jerkiness measure.
\color{black}
This is caused by the fact that robots can navigate smoothly over simple terrain types (e.g., grass, sand and gravel), but robots have to move slowly over large-rock terrain and the slow motion may reduce jerkiness.
TRAL and our approach achieve relatively similar jerkiness values, which significantly outperform LfD and MM-LfD.

\begin{figure}[htb]
  \centering
    \subfigure[Optimization hyperparameters]{
     \includegraphics[height=1.55in]{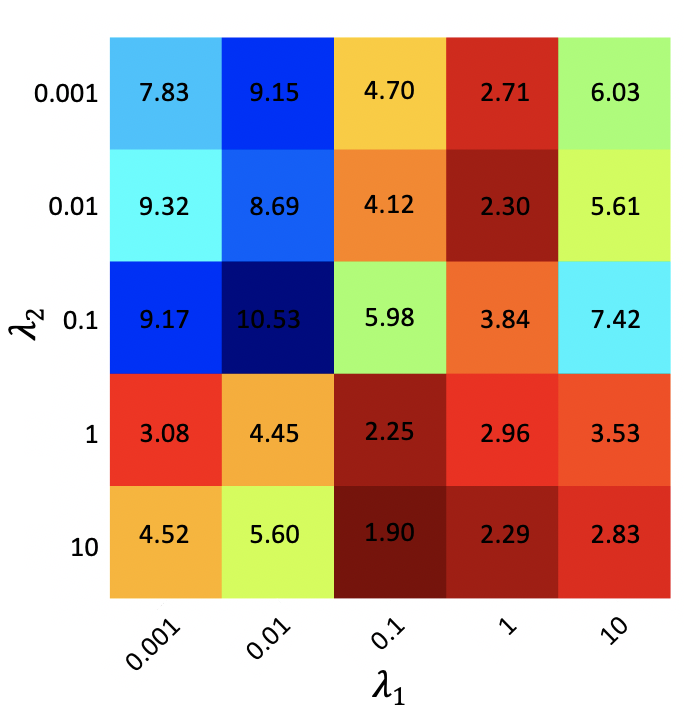}\label{lambda}
      }
    \subfigure[Sequence length]{
    \includegraphics[width=1.48in]{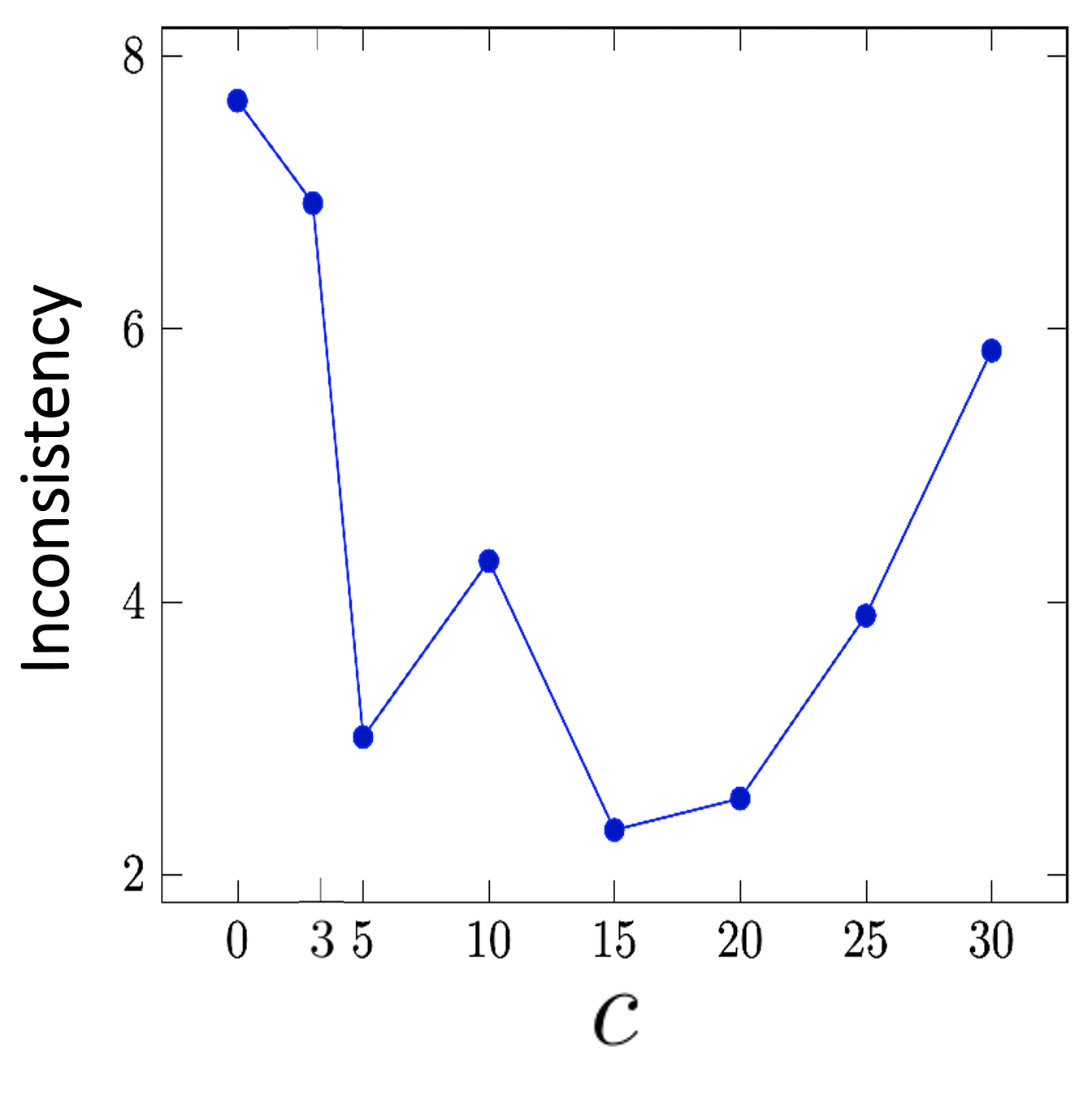}\label{curve_c}
      }
      \caption{Analysis of our approach based on the inconsistency metric.}
  \end{figure}

\subsection{Navigating over Complex Off-road Unstructured Terrains}

In the second set of experiments,
we evaluate our approach when the robot navigates over complex off-road unstructured terrain.
The tracks in these experiments either show transitions between
different terrain types (i.e., grass to large rocks and grass to medium rocks)
or a mixture of different terrain types in real off-road environments (i.e., Mixed Terrain I  and Mixed Terrain II),
as shown in Fig. \ref{Scenario_2}.
In the experiments, no additional training is performed and the previously trained model from individual types of unstructured terrains is used directly. 

Table \ref{tab:RTT2} presents the quantitative results obtained by our approach and the comparison with other methods. In this set of experiments, it is observed that each of the methods have a much higher failure rate in general,
especially over the Mixed Terrain II and the terrain with transitions from grass to large rocks.
Our approach and TRAL generally perform equally well and significantly outperform LfD and MM-LfD in terms of failure rate.
Similar to the experiments over individual types of terrain,
we observe that both LfD methods have much small traversal time for successful runs, but they have a significantly higher failure rate compared with TRAL and our approach.
Moreover, LfD and MM-LfD are outperformed with respect to inconsistency and jerkiness, with MM-LfD performing worst among all tested methods, especially on the jerkiness metric.
In these tested scenarios of more complex off-road terrain,
our approach clearly outperforms other methods and obtains the state-of-the-art performance in terms of consistency and jerkiness.


\subsection{Discussion}

\noindent \textbf{Hyperparameter Analysis:} The hyperparameters $\lambda_{1}$ and $\lambda_{2}$ in Eq. (\ref{final}) are implemented to balance the loss function and regularization terms.
Fig. \ref{lambda} depicts how the inconsistency metric changes given varying $\lambda$ values based on cross validation during training.
It is observed that $\lambda_{1} \in (0.1,10)$ and $\lambda_2 \in (1,10)$ result in good performance in general. The best result is obtained when $\lambda_{1}=0.1$ and $\lambda_{2}=10$. These values are used during our execution for all experiments.

\begin{wrapfigure}{RI}{0.23\textwidth}
\centering
\vspace{-6pt}
\includegraphics[width=0.975\linewidth]{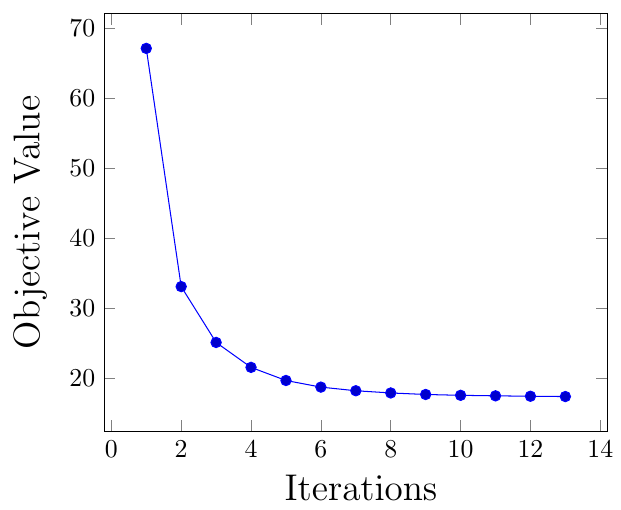}
\centering
\vspace{-9pt}
\caption{
Convergence. }
\label{fig:converge}
\vspace{-6pt}
\end{wrapfigure}

\noindent \textbf{Dependence on Frame Sequence:}
Our approach uses a sequence of historical frames with length $c$ to generate consistent behaviors. Fig. \ref{curve_c} shows the change of the inconsistency metric according to $c$.
It is observed that our approach generally performs well when $c \in (15, 20)$, and we observe that the inconsistency metric is worst when either a small number ($c < 5$) or a big number ($c>30$) is used under the sensing framerate of 30 Hz.

\noindent \textbf{Convergence:} Experimental results in Fig. \ref{fig:converge} illustrate the fast, monotonically decreasing convergence of Algorithm   \ref{alg1}, which validates our theoretical proof.

\section{Conclusion}\label{sec:CONC}

In this paper, we introduce a novel approach for consistent behavior generation that enables ground robots' actual behavior to more accurately match expected behaviors while adapting to a variety of unstructured off-road terrain.
Our approach learns offset behaviors to compensate for the inconsistency between the actual and expected behaviors without the need to explicitly model various setbacks,
and learns the importance of the multi-modal features to improve the representation of terrain for better adaptation. 
Additionally, we implement an optimization algorithm to solve the formulated problem with a theoretical guarantee.
Our proposed approach is extensively evaluated in real-world field environments.
Experimental results have shown the our approach enables robots to traverse complex unstructured off-road terrain with more consistent navigational behaviors and it outperforms previous methods.

\bibliographystyle{ieeetr}
\bibliography{v10}

\end{document}


\maketitle
\thispagestyle{empty}
\pagestyle{empty}

\section*{Proof of Convergence of the Optimization Algorithm}
In the following, we prove that Algorithm 1 in the main paper decreases the value of the objective function in Eq. (5) (of the main paper) with each iteration and converges to the global optimal solution. But first, we present a lemma:
\begin{lemma}\label{lemma1}
For any two given vectors $\mathbf{a}$ and $\mathbf{b}$, the following inequality relation holds:
$\|\mathbf{b}\|_2 - \frac{\|\mathbf{b}\|_2^2}{2\|\mathbf{a}\|_2}
\leq
\|\mathbf{a}\|_2 - \frac{\|\mathbf{a}\|_2^2}{2\|\mathbf{a}\|_2}$
\end{lemma}
\begin{proof}
\begin{equation}
-(\Vert\mathbf{b}\Vert_{2}-\Vert\mathbf{a}\Vert_{2})^2 \leq 0  \nonumber
\end{equation}
\begin{equation}
-\Vert\mathbf{b}\Vert_{2}^{2} - \Vert\mathbf{a}\Vert_{2}^{2} + 		        2\Vert\mathbf{b}\Vert_{2}\Vert\mathbf{a}\Vert_{2} \leq 0   \nonumber
\end{equation}
\begin{equation}
2\Vert\mathbf{b}\Vert_{2}\Vert\textbf{a}\Vert_{2} - \Vert\mathbf{b}\Vert_{2}^{2} \leq \Vert\textbf{a}\Vert_{2}^{2}       \nonumber
\end{equation}	
\begin{equation}
\Vert\mathbf{b}\Vert_{2} - \dfrac{\Vert\mathbf{b}\Vert_{2}^{2}}{2\Vert\textbf{a}\Vert_{2}} \leq  \Vert\textbf{a}\Vert_{2} - \dfrac{\Vert\textbf{a}\Vert_{2}^{2}}{2\Vert\textbf{a}\Vert_{2}}  \nonumber
\end{equation}
\end{proof}
From Lemma \ref{lemma1}, we can derive the following corollary:
\begin{corollary}\label{corollary1}
For any two given matrices $\mathbf{A}$ and $\mathbf{B}$ , the following inequality relation holds:
 \begin{eqnarray}
\Vert\mathbf{B}\Vert_F - \frac{\Vert\mathbf{B}\Vert_F^2}{2\Vert\mathbf{A}\Vert_F}
\leq
\Vert\mathbf{A}\Vert_F - \frac{\Vert\mathbf{A}\Vert_F^2}{2\Vert\mathbf{A}\Vert_F} \nonumber
\end{eqnarray}
\end{corollary}
\begin{theorem}\label{thm1}
Algorithm 1 converges to the converges to the global optimal solution to the optimization problem in Eq. (5) (of the main paper).
\end{theorem}

\begin{proof}
According to Step 4 of Algorithm 1, at iteration step $s$, the value of $W^{i}(s+1)$ can be given as:
\begin{eqnarray}\label{eq:proof1_1}
\mathbf{W}^{i}(s+1)&=&\Vert \mathbf{\hat{Y}}-\mathbf{W}^\top(s) \mathbf{X} - \mathbf{U}^\top(s)\mathbf{E}\Vert_{F}^{2}\nonumber \\
\quad \quad &+&\sum_{i=1}^{m}(\lambda_1 Tr(\mathbf{W}_{i})^\top\mathbf{Q}_{i}(s+1)(\mathbf{W}_{i}))
\end{eqnarray}
where $\mathbf{Q}_{i}(s+1) = \frac{1}{2\Vert \mathbf{W}_{i}(s)\Vert_{F}} \mathbf{I}_{q_i}$. Also, from Step 6 of Algorithm 1 we have:
\begin{eqnarray}\label{eq:proof1_2}
\mathbf{U}^{i}(s+1)&=&\Vert \mathbf{\hat{Y}}-\mathbf{W}^\top(s+1) \mathbf{X} - \mathbf{U}^\top(s)\mathbf{E}\Vert_{F}^{2}\nonumber \\
\quad \quad &+&\sum_{k=t}^{t-c}(\lambda_2 Tr(\mathbf{U}^{(k)})^\top\mathbf{P}^{(k)}(s+1)(\mathbf{U}^{(k)}))
\end{eqnarray}
where $\mathbf{P}^{(k)}(t+1) = \frac{1}{2\Vert \mathbf{U}^{(k)}\Vert_{F}} \mathbf{I}_{r}$.
Then, we can derive that
\begin{eqnarray}\label{eq:proof2}
&& \mathcal{J}(s+1) + \nonumber \\ &&\sum_{i=1}^{m}(\lambda_1 Tr(\mathbf{W}_{i}(s+1))^\top\mathbf{Q}_{i}(s+1)(\mathbf{W}_{i}(s+1))) \nonumber\\
&\leq & \mathcal{J}(s) + \sum_{i=1}^{m}(\lambda_1 Tr(\mathbf{W}_{i}(s))^\top\mathbf{Q}_{i}(s)(\mathbf{W}_{i}(s)))
\end{eqnarray}
where $\mathcal{J}(s)=\Vert \mathbf{\hat{Y}}-\mathbf{W}^\top(s) \mathbf{X} - \mathbf{U}^\top(s)\mathbf{E}\Vert_{F}^{2}$ 
After substituting the definition of each of $\mathbf{Q}^{i}$, we obtain
\begin{eqnarray}\label{eq:proof3}
&&\mathcal{J}(s+1)+ \sum_{i=1}^{m}(\lambda_1\dfrac{\|\mathbf{W}_{i}(s+1)\|_F^2}{2\|\mathbf{W}_{i}(s)\|_F}) \nonumber\\
&\leq & \mathcal{J}(s)+ \sum_{i=1}^{m}(\lambda_1\dfrac{\|\mathbf{W}_{i}(s)\|_F^2}{2\|\mathbf{W}_{i}(s)\|_F})
\end{eqnarray}
From Lemma \ref{lemma1} and Corollary \ref{corollary1}, we get $\forall i=1,\dots,m$
\begin{eqnarray}\label{eq:proof4}
&&\sum_{i=1}^{m}\Bigg({\|\mathbf{W}_{i}(s+1)\|_F} - {\dfrac{\|\mathbf{W}_{i}(s+1)\|_F^2}{2\|\mathbf{W}_{i}(s)\|_F}}\Bigg)\nonumber\\ &\leq&\sum_{i=1}^{m}\Bigg({\|\mathbf{W}_{i}(s)\|_F} - {\dfrac{\|\mathbf{W}_{i}(s)\|_F^2}{2\|\mathbf{W}_{i}(s)\|_F}}\Bigg).
\end{eqnarray}
Adding Eq. (\ref{eq:proof3}) and (\ref{eq:proof4}) on both sides, we have
\vspace{-6pt}
\begin{eqnarray}\label{eq:proof6}
&&\mathcal{J}(s+1)+ \sum_{i=1}^{m}(\lambda_1{\|\mathbf{W}_{i}(s+1)\|_F})\nonumber\\
&\leq&\mathcal{J}(s)+ \sum_{i=1}^{m}(\lambda_1{\|\mathbf{W}_{i}(s)\|_F})
\end{eqnarray}
Eq. (\ref{eq:proof6}) implies that when the value of $\mathbf{U}$ is kept a constant, the updated value of weight matrix $\mathbf{W}$, decreases the value of the objective function in each iteration. With the updated value of the weight matrix $\mathbf{W}$, we can derive that:
\begin{eqnarray}\label{eq:proof2u}
&& \mathcal{F}(s+1) + \nonumber \\ &&\sum_{k=t}^{t-c}(\lambda_2 Tr(\mathbf{U}^{(k)}(s+1))^\top\mathbf{P}^{(k)}(s+1)(\mathbf{U}^{(k)}(s+1))) \nonumber\\
&\leq & \mathcal{F}(s) + \sum_{k=t}^{t-c}(\lambda_2 Tr(\mathbf{U}^{(k)}(s))^\top\mathbf{P}^{(k)}(s)(\mathbf{U}^{(k)}(s)))
\end{eqnarray}
where $\mathcal{F}(s)=\Vert \mathbf{\hat{Y}}-\mathbf{W}^\top(s+1) \mathbf{X} - \mathbf{U}^\top(s)\mathbf{E}\Vert_{F}^{2}$.
After substituting the definition of each of $\mathbf{P}^{(k)}$ in Eq. (\ref{eq:proof2u}), we obtain
\begin{eqnarray}\label{eq:proof3u}
&&\mathcal{F}(s+1)+ \sum_{k=t}^{t-c}(\lambda_2\dfrac{\|\mathbf{U}^{(k)}(s+1)\|_F^2}{2\|\mathbf{Q}^{(k)}(s)\|_F}) \nonumber\\
&\leq & \mathcal{F}(s)+ \sum_{k=t}^{t-c}(\lambda_2\dfrac{\|\mathbf{U}^{(k)}(s)\|_F^2}{2\|\mathbf{U}^{(k)}(s)\|_F})
\end{eqnarray}
Similar to Eq. (\ref{eq:proof4}), from Lemma \ref{lemma1} and Corollary \ref{corollary1}, we get $\forall k=t,\dots,t-c$
\begin{eqnarray}\label{eq:proof4u}
&&\sum_{k=t}^{t-c}\Bigg({\|\mathbf{U}^{(k)}(s+1)\|_F} - {\dfrac{\|\mathbf{U}^{(k)}(s+1)\|_F^2}{2\|\mathbf{U}^{(k)}(s)\|_F}}\Bigg)\nonumber\\ &\leq&\sum_{k=t}^{t-c}\Bigg({\|\mathbf{U}^{(k)}(s)\|_F} - {\dfrac{\|\mathbf{U}^{(k)}(s)\|_F^2}{2\|\mathbf{U}^{(k)}(s)\|_F}}\Bigg).
\end{eqnarray}
Adding Eq. (\ref{eq:proof3u}) and Eq. (\ref{eq:proof4u}) on both sides we obtain
\begin{eqnarray}\label{eq:proof7}
&&\mathcal{F}(s+1)+ \sum_{k=t}^{t-c}(\lambda_2{\|\mathbf{U}^{(k)}(s+1)\|_F})\nonumber\\
&\leq&\mathcal{F}(s)+ \sum_{k=t}^{t-c}(\lambda_2{\|\mathbf{U}^{(k)}(s)\|_F})
\end{eqnarray}
Eq. (\ref{eq:proof7}) implies that for a fixed value of $\mathbf{W}$, the updated value of weight matrix $\mathbf{U}$, decreases the value of the objective function in each iteration. A single iteration of Algorithm 1 involves both Eq. (\ref{eq:proof6}) and Eq. (\ref{eq:proof7}). Therefore, we add these equations on both sides to obtain:
\begin{eqnarray}\label{eq:proof8}
&&\mathcal{J}(s+1)+ \mathcal{F}(s+1)+ \|\mathbf{W}(s+1)\|_M + \|\mathbf{U}(s+1)\|_T\nonumber\\
&&\leq\mathcal{J}(s)+ \mathcal{F}(s)+ \|\mathbf{W}(s)\|_M + \|\mathbf{U}(s)\|_T
\end{eqnarray}
The loss function $\mathcal{J}$ contains values of loss function without the updated weight matrix $\mathbf{U}$ and can be dropped. Also, $\mathcal{J}$ only denotes the value of the loss function at the intermediate Step 4 of Algorithm 1, whereas, $\mathcal{F}$ is the value of loss at the end of each iteration. Then we have
\begin{eqnarray}\label{eq:proof9}
&&\mathcal{F}(s+1)+ \|\mathbf{W}(s+1)\|_M + \|\mathbf{U}(s+1)\|_T\nonumber\\
&&\leq \mathcal{F}(s)+ \|\mathbf{W}(s)\|_M + \|\mathbf{U}(s)\|_T
\end{eqnarray}

Eq. (\ref{eq:proof9}) decreases the value of the objective function with each iteration. As our objective function is convex, Algorithm \ref{thm1} converges to the global optimal value. Therefore, Algorithm 1 converges to the optimal solution to the optimization problem in Eq. (5) of the main paper.
\end{proof}

\section*{Derivation for Predicting Future Offset Behaviors}

Future offsets are predicted by predicting the robot navigational behaviors one step ahead of time. 
To start with predicting the future behavior difference, we make the following assumptions: For the most recent time step $k=t$, the weights corresponding to the behavior difference and features are set as $\mathbf{u}^{j(k)}=1$ and $\mathbf{w}_i^{j(k)}= 0$ respectively. It is to be noted that it is not a restriction as for any value of weights $\mathbf{u}^{j(k)}$ different from '1', it can directly be incorporated as a difference in behaviors $\mathbf{e}^{k}$. Setting $\mathbf{w}_i^{j(k)}= 0$  makes use of only the previous terrain features to estimate the future behavior difference. 

We know that the offset behaviors $\mathbf{v}_{i}$ are computed as $\mathbf{v}_{i} = \mathbf{U}^{\top}\mathbf{e}_{i}$, and the actual behaviors are given as $\hat{\mathbf{y}_{i}} = \mathbf{y}_{i} + \mathbf{v}_{i}$. The value of $\hat{\mathbf{y}_{i}}^{k}$ and $\mathbf{y}_{i}^{k}$,  $\forall k= (t-1),\dots,(t-c)$; are obtained from a history of previous $c$-time steps. Thus the behavior difference at each of the previous $c$-time steps can be given as:
\begin{equation}
    \mathbf{v}_{i}^{k} = \hat{\mathbf{y}_{i}}^{k}-\mathbf{y}_{i}^{k}
\end{equation}

At the current time step, before the robot has executed its navigational behaviors, offset behaviors are not known and we take it's estimate as $\hat{\mathbf{v}}^{(t \vert t-1)}$. As we take the assumption that at $k=t, \mathbf{w}_{i}^{j(k)} = 0$, theoretically the actual navigational behaviors at current time $\hat{\mathbf{y}}_{i}^{t}$ doesn't depend on the current value of the terrain features $\mathbf{x}^{(t)}$. Then the estimate for the current navigational behaviors $\tilde{y}^{(t \vert t-1)}$  can be given as
\begin{equation}\label{eq3}
\tilde{\mathbf{y}}^{(t \vert t-1)} = \sum_{k=t-1}^{t-c}\sum_{j=1}^{r}\sum_{i=1}^{m}\mathbf{w}_{i}^{j(k)}^{\top}\mathbf{x}_{i}^{j(k)} + \hat{\mathbf{v}}^{(t \vert t-1)}
\end{equation}
Then the best prediction of offset behaviors can be given via its expectation as: 
\begin{equation}
\hat{\mathbf{v}}^{(t \vert t-1)} = \mathbb{E}(\mathbf{e}^{(t)}) + \mathbb{E}( \sum_{k=t-1}^{t-c} \sum_{j=1}^{r}\mathbf{u}^{j(k)}^{\top}\mathbf{e}^{(k)})
\end{equation}

The behavior differences $\mathbf{e}$,  depend on both the environment and the robot. The behavior difference at each time step is random with zero mean. Then the following equality hold true:
\begin{eqnarray}
\hat{\mathbf{v}}^{(t \vert t-1)} =  \mathbb{E}( \sum_{k=t-1}^{t-c} \sum_{j=1}^{r}\mathbf{u}^{j(k)}^{\top}\mathbf{e}^{j(k)})
\end{eqnarray}
\begin{eqnarray}
\hat{\mathbf{v}}^{(t \vert t-1)} =  (\mathbf{U}^{\top} - \mathbf{1})(\mathbf{e}^{(k)})
\end{eqnarray}
where $\mathbf{1}$ is an matrix of ones, that is each element is one. 
Substituting the values of $\mathbf{e}^{(k)}$, in the above equation, we obtain
\begin{eqnarray}
\hat{\mathbf{v}}^{(t \vert t-1)}  =  (\mathbf{U}^{\top} - \mathbf{1})(\mathbf{U}^{\top})^{-1}\mathbf{v}^{(k)}
\end{eqnarray}
Substituting this value of $\hat{\mathbf{v}}^{(t \vert t-1)} $ in Eq. (\ref{eq3}), we obtain
\begin{eqnarray}\label{eq5}
\tilde{\mathbf{y}}^{(t \vert t-1)} = \sum_{k=t-1}^{t-c}\sum_{j=1}^{r}\sum_{i=1}^{m}\mathbf{w}_{i}^{j(k)}^{\top}\mathbf{x}_{i}^{j(k)} \nonumber \\ + (\mathbf{1}- (\mathbf{U}^{\top})^{-1})\mathbf{v}^{(k)}
\end{eqnarray}
Substituting the value of $\mathbf{v}^{(k)}$ while calculating $\tilde{\mathbf{e}}^{(t)} =  \hat{\mathbf{y}}^{(t)} - \tilde{\mathbf{y}}^{(t \vert t-1)}$ and also $\mathbf{u}^{j(k)}=1$, we obtain the predicted offset behaviors as:
\begin{equation}
\tilde{\mathbf{v}}^{(t)}  =\sum_{k=t}^{t-c}\Big(\big(\mathbf{U}^{(k)\top}\big)^{-1}\big(\mathbf{y}^{(k)} - \mathbf{W}^{(k)\top}\mathbf{x}^{(k)}\big) \Big)
\end{equation}

This predicted offset behaviors take into account the future behavior differences to help in achieving consistent behaviors.